%% file: Paper.tex
\documentclass[letterpaper]{article}
\usepackage{uai2019}
\usepackage[margin=1in]{geometry}

\usepackage{algorithm}
\usepackage{algorithmicx}
\usepackage[noend]{algpseudocode}
\usepackage{amsmath}
\usepackage{amssymb}
\usepackage{amsthm}
\usepackage{bbm}
\usepackage{bm}
\usepackage{color}
\usepackage{dirtytalk}
\usepackage{dsfont}
\usepackage{enumerate}
\usepackage{graphicx}
\usepackage{mathtools}
\usepackage[numbers]{natbib}
\usepackage{subfigure}
\usepackage{times}

\usepackage[usenames,dvipsnames]{xcolor}
\usepackage[bookmarks=false]{hyperref}
\hypersetup{
  pdftex,
  pdffitwindow=true,
  pdfstartview={FitH},
  pdfnewwindow=true,
  colorlinks,
  linktocpage=true,
  linkcolor=Green,
  urlcolor=Green,
  citecolor=Green
}
\usepackage[capitalize,noabbrev]{cleveref}

\newcommand{\commentout}[1]{}
\newcommand{\junk}[1]{}

\Crefname{corollary}{Corollary}{Corollaries}
\Crefname{proposition}{Proposition}{Propositions}
\Crefname{theorem}{Theorem}{Theorems}
\Crefname{definition}{Definition}{Definitions}
\Crefname{assumption}{Assumption}{Assumptions}
\Crefname{example}{Example}{Examples}
\Crefname{remark}{Remark}{Remarks}
\Crefname{setting}{Setting}{Settings}
\Crefname{lemma}{Lemma}{Lemmas}

\usepackage{thmtools}
\declaretheorem[name=Theorem,refname={Theorem,Theorems},Refname={Theorem,Theorems}]{theorem}
\declaretheorem[name=Lemma,refname={Lemma,Lemmas},Refname={Lemma,Lemmas},sibling=theorem]{lemma}

\newcommand{\cF}{\mathcal{F}}

\newcommand{\cH}{\mathcal{H}}
\newcommand{\eps}{\varepsilon}

\newcommand{\realset}{\mathbb{R}}

\newcommand{\E}[1]{\mathbb{E} \left[#1\right]}
\newcommand{\condE}[2]{\mathbb{E} \left[#1 \,\middle|\, #2\right]}
\newcommand{\Et}[1]{\mathbb{E}_t \left[#1\right]}
\newcommand{\prob}[1]{\mathbb{P} \left(#1\right)}
\newcommand{\condprob}[2]{\mathbb{P} \left(#1 \,\middle|\, #2\right)}
\newcommand{\probt}[1]{\mathbb{P}_t \left(#1\right)}

\newcommand{\condvar}[2]{\mathrm{var} \left[#1 \,\middle|\, #2\right]}

\newcommand{\abs}[1]{\left|#1\right|}
\newcommand{\ceils}[1]{\left\lceil#1\right\rceil}
\newcommand{\dbar}[1]{\bar{\bar{#1}}}

\newcommand{\I}[1]{\mathds{1} \! \left\{#1\right\}}

\newcommand{\norm}[1]{\|#1\|}
\newcommand{\normw}[2]{\|#1\|_{#2}}

\newcommand{\set}[1]{\left\{#1\right\}}

\newcommand{\T}{^\top}
\newcommand{\xmin}[1]{\langle #1\rangle}

\DeclareMathOperator*{\argmax}{arg\,max\,}
\DeclareMathOperator*{\argmin}{arg\,min\,}

\mathchardef\mhyphen="2D

\newcommand{\glmucb}{{\tt GLM\mhyphen UCB}}

\newcommand{\linphe}{{\tt LinPHE}}
\newcommand{\lints}{{\tt LinTS}}
\newcommand{\linucb}{{\tt LinUCB}}
\newcommand{\logphe}{{\tt LogPHE}}
\newcommand{\logts}{{\tt LogTS}}

\newcommand{\ucbglm}{{\tt UCB\mhyphen GLM}}

\title{Perturbed-History Exploration in Stochastic Linear Bandits}

\author{
Branislav Kveton \\ Google Research
\And
Csaba Szepesv\'ari \\ DeepMind
\And
Mohammad Ghavamzadeh \\ Facebook AI Research
\And
Craig Boutilier \\ Google Research
}

\begin{document}

\maketitle

\begin{abstract}
We propose a new online algorithm for cumulative regret minimization in a stochastic linear bandit. The algorithm pulls the arm with the highest estimated reward in a linear model trained on its perturbed history. Therefore, we call it \emph{perturbed-history exploration in a linear bandit ($\linphe$)}. The \emph{perturbed history} is a mixture of observed rewards and randomly generated \emph{i.i.d.\ pseudo-rewards}. We derive a $\tilde{O}(d \sqrt{n})$ gap-free bound on the $n$-round regret of $\linphe$, where $d$ is the number of features. The key steps in our analysis are new concentration and anti-concentration bounds on the weighted sum of Bernoulli random variables. To show the generality of our design, we generalize $\linphe$ to a logistic model. We evaluate our algorithms empirically and show that they are practical.
\end{abstract}

\input{Introduction}

\input{Setting}

\input{LinPHE}

\input{Analysis}

\input{Experiments}

\input{RelatedWork}

\input{Conclusions}

\input{Appendix}

\newpage

\bibliographystyle{plainnat}
\bibliography{References}

\end{document}

%% file: Introduction.tex

\section{INTRODUCTION}
\label{sec:introduction}

A \emph{multi-armed bandit} \cite{lai85asymptotically,auer02finitetime,lattimore19bandit} is an online learning problem where the \emph{learning agent} acts by pulling \emph{arms}. After the arm is \emph{pulled}, the agent receives its \emph{stochastic reward}. The objective of the agent is to maximize its expected cumulative reward. The agent does not know the mean rewards of the arms in advance and faces the so-called \emph{exploration-exploitation dilemma}: \emph{explore}, and learn about arms; or \emph{exploit}, and pull the arm with the highest estimated reward thus far. This model captures many practical applications. In a clinical trial, for example, the \emph{arm} may be a treatment and its \emph{reward} is the outcome of that treatment on some patient population.

A \emph{stochastic linear bandit} \cite{dani08stochastic,rusmevichientong10linearly,abbasi-yadkori11improved} is a generalization of the multi-armed bandit to the setting where each arm is associated with a feature vector. The mean reward of an arm is the dot product of its feature vector and an unknown parameter vector, which is shared by all arms. In our clinical example, the feature vector may be a vector of treatment indicators and the parameter vector may be the effects of individual treatments.

The most popular exploration strategies in stochastic bandits, \emph{optimism in the face of uncertainty (OFU)} \cite{auer02finitetime} and \emph{Thompson sampling (TS)} \cite{thompson33likelihood,agrawal13further,russo18tutorial}, are relatively well understood in linear bandits \cite{dani08stochastic,abbasi-yadkori11improved,agrawal13thompson,lattimore19bandit}. Unfortunately, these designs and their guarantees do not extend easily to complex problems. For instance, in generalized linear bandits \cite{filippi10parametric}, all OFU algorithms use \emph{approximate} high-probability confidence sets, which are loose and statistically suboptimal \cite{filippi10parametric,li17provably,jun17scalable}. Also the posterior distribution of model parameters does not have a closed form. Therefore, posterior sampling in TS has to be \emph{approximated}. Posterior approximations are computationally costly in general \cite{gopalan14thompson,kawale15efficient,lu17ensemble,riquelme18deep,lipton18bbq,liu18customized}.

In this work, we study a simple exploration strategy that can be easily generalized to complex problems. The key idea is to \emph{explore by perturbing} the training data of a reward generalization model, which is fit by an existing offline oracle. Specifically, the model is fit to a mixture of \emph{history}, features of the pulled arms with their realized rewards; and \emph{pseudo-history}, features of the pulled arms with randomly generated \emph{i.i.d.\ pseudo-rewards}. In \emph{perturbed-history exploration (PHE)}, the agent pulls the arm with the highest reward in its estimated model and then updates its history with the observed reward.

The key to the generality and optimism in PHE are the pseudo-rewards. They are drawn from the same family of distributions as the actual rewards, and thus we can reuse existing methods for fitting the reward generalization model. They are also \emph{maximum variance randomized data}, which induce suitable exploration. We show that appropriate randomization, not necessarily by posterior sampling, can lead to practical exploration in structured problems.

We make the following contributions in this paper. First, we propose $\linphe$, a linear bandit algorithm that estimates the mean rewards of arms using PHE. Second, we prove a $\tilde{O}(d \sqrt{n})$ gap-free bound on the $n$-round regret of $\linphe$, where $d$ is the number of features. Our analysis relies on novel concentration and anti-concentration bounds on the weighted sum of Bernoulli random variables. Third, we propose a generalization of $\linphe$ to a logistic model and call it $\logphe$. Finally, we evaluate both algorithms empirically. They are competitive with Thompson sampling, although they are derived based on different insights.

%% file: Setting.tex

\section{SETTING}
\label{sec:setting}

We adopt the following notation. The set $\set{1, \dots, n}$ is denoted by $[n]$. All vectors are column vectors. The minimum and maximum eigenvalues of matrix $M$ are denoted by $\lambda_{\min}(M)$ and $\lambda_{\max}(M)$, respectively. We define $\mathrm{Ber}(x; p) = p^x (1 - p)^{1 - x}$ and let $\mathrm{Ber}(p)$ be the corresponding Bernoulli distribution. We also define $B(x; n, p) = \binom{n}{x} p^x (1 - p)^{n - x}$ and let $B(n, p)$ be the corresponding binomial distribution. For any event $E$, $\I{E} = 1$ if event $E$ occurs and $\I{E} = 0$ otherwise. We denote a $d \times d$ identity matrix by $I_d$. We use $\tilde{O}$ for the big-O notation up to logarithmic factors.

A \emph{stochastic linear bandit} \cite{dani08stochastic,rusmevichientong10linearly,abbasi-yadkori11improved} is an online learning problem where the learning agent sequentially pulls arms, each of which is associated with a feature vector. Let $K$ be the number of arms, $x_i \in \realset^d$ be the \emph{feature vector} of arm $i \in [K]$, and $\theta_\ast \in \realset^d$ be an unknown \emph{parameter vector}. The \emph{reward} of arm $i$ in round $t \in [n]$, $Y_{i, t}$, is drawn i.i.d.\ from a distribution of that arm with mean $\mu_i = x_i\T \theta_\ast$. The learning agent acts as follows. In round $t$, it \emph{pulls} arm $I_t \in [K]$ and receives reward $Y_{I_t, t}$. The agent aims to maximize its \emph{expected cumulative reward} in $n$ rounds. To simplify exposition, we denote by $X_t = x_{I_t}$ and $Y_t = Y_{I_t, t}$ the feature vector of the pulled arm in round $t$ and its reward, respectively.

Without loss of generality, we assume that arm $1$ is \emph{optimal}, that is $\mu_1 > \max_{i > 1} \mu_i$. Let $\Delta_i = \mu_1 - \mu_i$ denote the \emph{gap} of arm $i$. Maximization of the expected cumulative reward in $n$ rounds is equivalent to minimizing the \emph{expected $n$-round regret},
\begin{align*}
  R(n)
  = \sum_{i = 2}^K \Delta_i \E{\sum_{t = 1}^n \I{I_t = i}}\,.
\end{align*}
We make several additional assumptions. First, rewards are bounded in $[0, 1]$, that is $Y_{i, t} \in [0, 1]$ for any arm $i$ and round $t$. This assumption is standard. Second, the last feature is a \emph{bias term}, $x_i(d) = 1$ for all arms $i$. This is without loss of generality, since such a feature can be always added. Finally, the feature vectors of the last $d$ arms are a basis in $\realset^d$. This is without loss of generality, since the arms can be always reordered to satisfy this.

%% file: LinPHE.tex

\section{PERTURBED-HISTORY EXPLORATION}
\label{sec:linphe}

\begin{algorithm}[t]
  \caption{Perturbed-history exploration in a linear bandit ($\linphe$) with $[0, 1]$ rewards.}
  \label{alg:linphe}
  \begin{algorithmic}[1]
    \State \textbf{Inputs}:
    \State \quad Integer perturbation scale $a > 0$
    \State \quad Regularization parameter $\lambda > 0$
    \Statex 
    \For{$t = 1, \dots, n$}
      \If{$t > d$}
        \State Generate $(Z_{j, \ell})_{j \in [a], \, \ell \in [t - 1]} \sim
        \mathrm{Ber}(1 / 2)$
        \label{alg:linphe:pseudo-rewards}
        \State $\displaystyle
        G_t \gets (a + 1) \sum_{\ell = 1}^{t - 1} X_\ell X_\ell\T +
        \lambda (a + 1) I_d$
        \State $\displaystyle
        \tilde{\theta}_t \gets G_t^{-1} \sum_{\ell = 1}^{t - 1} X_\ell
        \Bigg[Y_\ell + \sum_{j = 1}^a Z_{j, \ell}\Bigg]$
        \label{alg:linphe:value}
        \State $I_t \gets \argmax_{i \in [K]} x_i\T \tilde{\theta}_t$
        \label{alg:linphe:pulled arm}
      \Else
        \State $I_t \gets K - t + 1$
        \label{alg:linphe:initialization}
      \EndIf
      \State Pull arm $I_t$ and get reward $Y_{I_t, t}$
      \State $X_t \gets x_{I_t}, \ Y_t \gets Y_{I_t, t}$
    \EndFor
  \end{algorithmic}
\end{algorithm}

Now we introduce \emph{perturbed-history exploration (PHE)}. Our algorithm, \emph{perturbed-history exploration in a linear bandit} ($\linphe$), is presented in \cref{alg:linphe}. In round $t$, $\linphe$ fits a linear model to its \emph{perturbed history} up to round $t$ (line \ref{alg:linphe:value}),
\begin{align}
  \tilde{\theta}_t
  = G_t^{-1} \sum_{\ell = 1}^{t - 1} X_\ell
  \Bigg[Y_\ell + \sum_{j = 1}^a Z_{j, \ell}\Bigg]\,,
  \label{eq:theta tilde}
\end{align}
where
\begin{align}
  G_t
  = (a + 1) \sum_{\ell = 1}^{t - 1} X_\ell X_\ell\T + \lambda (a + 1) I_d
  \label{eq:sample covariance matrix}
\end{align}
is the \emph{sample covariance matrix} up to round $t$, $a > 0$ is a tunable integer parameter, $\lambda > 0$ is the regularization parameter, and $(Z_{j, \ell})_{j \in [a], \, \ell \in [t - 1]}$ are \emph{i.i.d.\ pseudo-rewards}, which are freshly sampled in each round. Our model can be viewed as follows. If $Z_{j, \ell}$ were omitted in \eqref{eq:theta tilde} and $a + 1$ was omitted in \eqref{eq:sample covariance matrix}, we would get a regularized least-squares regression on rewards up to round $t$. Thus, $\tilde{\theta}_t$ is a regularized least-squares solution on the past $t - 1$ rewards and \emph{$a (t - 1)$ i.i.d.\ pseudo-rewards}.

$\linphe$ pulls the arm with the highest estimated reward under $\tilde{\theta}_t$ (line \ref{alg:linphe:pulled arm}). Any tie-breaking rule can be used as needed. $\linphe$ is initialized by pulling each arm in the basis once (line \ref{alg:linphe:initialization}). This guarantees that $\linphe$ is sufficiently optimistic about any optimal arm (\cref{lem:theta tilde anti-concentration}).

$\linphe$ has two tunable parameters. The \emph{perturbation scale $a$} dictates the number of pseudo-rewards for each observed reward in the perturbed history. Therefore, it trades off exploration and exploitation, with higher values of $a$ leading to more exploration. We argue informally in \cref{sec:informal justification} that any $a > 1$ is sufficient for sublinear regret. The formal regret analysis is in \cref{sec:analysis}. The \emph{regularization parameter} $\lambda > 0$ ensures that $G_t$ can be inverted and makes $\linphe$ stable. Regularization is used frequently in linear bandit analyses \cite{abbasi-yadkori11improved,agrawal13thompson}.

\subsection{Informal Justification}
\label{sec:informal justification}

Before we analyze $\linphe$ in \cref{sec:analysis}, we informally explain how exploration arises in it. To do this, we introduce two least-squares solutions that are closely related to $\tilde{\theta}_t$ in \eqref{eq:theta tilde}. In the first, the pseudo-rewards are replaced by their means,
\begin{align}
  \bar{\theta}_t
  = G_t^{-1} \sum_{\ell = 1}^{t - 1} X_\ell
  \Bigg[Y_\ell + \sum_{j = 1}^a \bar{Z}_{j, \ell}\Bigg]\,,
  \label{eq:theta hat}
\end{align}
where $\bar{Z}_{j, \ell} = \E{Z_{j, \ell}} = 1 / 2.$ In the second, both the rewards 
and pseudo-rewards are the so-replaced,
\begin{align*}
  \dbar{\theta}_t
  = G_t^{-1} \sum_{\ell = 1}^{t - 1} X_\ell
  \Bigg[X_\ell\T \theta_\ast + \sum_{j = 1}^a \bar{Z}_{j, \ell}\Bigg]\,.
\end{align*}
Let $\cH = (I_1, \dots, I_{t - 1})$ be a sequence of pulled arms in the first $t - 1$ rounds.

The solution $\tilde{\theta}_t$ has two important properties that allow us to bound the regret of $\linphe$. First, it concentrates at $\dbar{\theta}_t$ given history $\cH$, since $\dbar{\theta}_t$ solves a noiseless variant of the least-squares problem solved by $\tilde{\theta}_t$. Furthermore, $\dbar{\theta}_t \to \theta'$ as regularization vanishes, where $\theta'$ are scaled and shifted parameters of the original problem. That is, $x_i\T \theta' = (\mu_i + a / 2) / (a + 1)$ for all arms $i$.

Second, from the definitions of $\tilde{\theta}_t$, $\bar{\theta}_t$, and $\dbar{\theta}_t$, we have
\begin{align*}
  x_i\T \dbar{\theta}_t - x_i\T \bar{\theta}_t
  & = x_i\T G_t^{-1} \sum_{\ell = 1}^{t - 1} 
  X_\ell W_\ell\,, \\
  x_i\T \tilde{\theta}_t - x_i\T \bar{\theta}_t
  & = x_i\T G_t^{-1} \sum_{\ell = 1}^{t - 1} X_\ell
  \sum_{j = 1}^a (Z_{j, \ell} - \bar{Z}_{j, \ell})\,,
\end{align*}
where $W_\ell = X_\ell\T \theta_\ast - Y_\ell$ is the \say{noise} in the reward in round $\ell$. The first term is the deviation in the estimated reward of arm $i$ due to reward randomness. The second term represents the deviation in the estimated reward of arm $i$ due to pseudo-reward randomness.

Fix history $\cH$ and let $(Y_\ell)_{\ell = 1}^{t - 1}$ be conditionally independent given $\cH$. Then
\begin{align*}
  \condvar{x_i\T \dbar{\theta}_t - x_i\T \bar{\theta}_t}{\cH}
  < \condvar{x_i\T \tilde{\theta}_t - x_i\T \bar{\theta}_t}{\cH}
\end{align*}
for $a > 1$, because $x_i\T \dbar{\theta}_t - x_i\T \bar{\theta}_t \mid \cH$ is a weighted sum of $t - 1$ i.i.d.\ reward deviations and $x_i\T \tilde{\theta}_t - x_i\T \bar{\theta}_t \mid \cH$ is a weighted sum, with the same weights, of $a (t - 1)$ i.i.d.\ maximum-variance deviations on $[0, 1]$.

If both $x_i\T \dbar{\theta}_t - x_i\T \bar{\theta}_t$ and $x_i\T \tilde{\theta}_t - x_i\T \bar{\theta}_t$ were normally distributed, this would imply that for any $\varepsilon > 0$,
\begin{align*}
  & \condprob{x_i\T \dbar{\theta}_t -
  x_i\T \bar{\theta}_t = \eps}{\cH} \\
  & \quad \leq \condprob{x_i\T \dbar{\theta}_t -
  x_i\T \bar{\theta}_t \geq \eps}{\cH} \\
  & \quad < \condprob{x_i\T \tilde{\theta}_t -
  x_i\T \bar{\theta}_t \geq \eps}{\cH}\,,
\end{align*}
where the first inequality holds trivially. That is, for any potentially harmful deviation $\varepsilon > 0$ in the estimated reward of arm $i$, $x_i\T \tilde{\theta}_t$ overestimates the perturbed mean reward with a higher probability than the probability of that deviation. This \emph{optimism} induces exploration and is the key feature of $\linphe$.

The idea of offsetting a fixed history of rewards by i.i.d.\ pseudo-rewards is very general and applies beyond the linear model in this section. In \cref{sec:algorithm logphe}, we apply it to a logistic model. In \cref{sec:experiments}, we evaluate our linear and logistic algorithms empirically.

\subsection{Efficient Implementation}
\label{sec:efficient implementation}

$\linphe$ can be implemented such that its expected computational cost in round $t$ is independent of $t$. In particular, line \eqref{alg:linphe:value} in $\linphe$ can be rewritten as
\begin{align}
  \tilde{\theta}_t
  = G_t^{-1} \sum_{i = 1}^K x_i [V_{i, t} + U_{i, t}]\,,
  \label{eq:efficient theta tilde}
\end{align}
where $V_{i, t}$ is the cumulative reward of arm $i$ in the first $t - 1$ rounds, $U_{i, t} \sim B(a \, T_{i, t - 1}, 1 / 2)$ is the sum of the pseudo-rewards of arm $i$ in round $t$, and $T_{i, t}$ is the number of pulls of arm $i$ in the first $t$ rounds. The statistics $V_{i, t}$ and $G_t$ can be updated incrementally as
\begin{align*}
  V_{i, t}
  & = V_{i, t - 1} + \I{I_{t - 1} = i} Y_{t - 1}\,, \\
  G_t
  & = G_{t - 1} + (a + 1) X_{t - 1} X_{t - 1}\T\,,
\end{align*}
where we assume that $V_{i, 0} = 0$ and $G_0 = \lambda (a + 1) I_d$. The inverse of $G_t$ can be also updated directly using the Sherman-Morrison formula.

The statistics $V_{i, t}$ and $G_t$ take $O(K + d^2)$ space. If the Sherman-Morrison formula is used, the cost of updating $G_t^{-1}$ is $O(d^2)$. After that, the cost of computing $\tilde{\theta}_t$ in \eqref{eq:efficient theta tilde} is $O(K d^2)$, if $U_{i, t}$ can be sampled in $O(1)$ time. Based on Section 4.4 of \citet{devroye86nonuniform}, $B(n, p)$ can be sampled from in $O(1)$ time in expectation for any $n$ and $p$.

\subsection{Algorithm $\logphe$}
\label{sec:algorithm logphe}

While our formal analysis is for linear bandits, the idea of PHE is much more general. To illustrate it, we extend $\linphe$ to a \emph{logistic bandit}, where the mean reward of arm $i$ is $\mu_i = \sigma(x_i\T \theta_\ast)$ and $\sigma(v) = 1 / (1 + \exp[- v])$ is a \emph{sigmoid} function. The reward of arm $i$ in round $t$ is drawn i.i.d.\ from $\mathrm{Ber}(\mu_i)$.

To extend $\linphe$ to this class of problems, we replace $\tilde{\theta}_t$ in $\linphe$ with the minimizer of
\begin{align*}
  \sum_{\ell = 1}^{t - 1} \Bigg[g(X_\ell\T \theta, Y_\ell) +
  \sum_{j = 1}^a g(X_\ell\T \theta, Z_{j, \ell})\Bigg] +
  \lambda \norm{\theta}_2^2\,,
\end{align*}
where $g(s, y) = - y \log(\sigma(s)) - (1 - y) \log(1 - \sigma(s))$. For $\lambda = 0$, we obtain the maximum likelihood solution.

The above problem is convex. Also the sufficient statistics in this problem, the number of positive and negative observations of arms, can be updated incrementally as in \cref{sec:efficient implementation}. Therefore, $\tilde{\theta}_t$ in round $t$ can be estimated in a constant time in $t$. We call this algorithm $\logphe$ and evaluate it in \cref{sec:logistic bandit experiments}.

%% file: Analysis.tex

\section{ANALYSIS}
\label{sec:analysis}

We now provide a formal analysis of $\linphe$. In \cref{sec:notation}, we introduce relevant notation. In \cref{sec:regret bound}, we prove a generic regret bound that applies to any randomized algorithm that estimates $\theta_\ast$. The regret bound of $\linphe$ in \cref{sec:discussion} is an instance of this result.

\subsection{Notation}
\label{sec:notation}

To simplify the analysis of $\linphe$, we assume that its sample covariance matrix is not scaled by $a + 1$. That is, $G_t = \sum_{\ell = 1}^{t - 1} X_\ell X_\ell\T + \lambda I_d$. This does not change the behavior of $\linphe$. We also assume that $\theta_\ast \in \realset^d$ is a parameter vector such that $x_i\T \theta_\ast = \mu_i + a / 2$ holds for any arm $i$. Note that this transformation does not change the gaps of arms. It only shifts their mean rewards by a factor of $a / 2$. Recall that arm $1$ is optimal.

Let $\cF_t = \sigma(I_1, \dots, I_t, Y_{I_1, 1}, \dots, Y_{I_t, t})$ be the $\sigma$-algebra generated by the pulled arms and their rewards by the end of round $t \in [n] \cup \set{0}$. We define $\cF_0 = \set{\emptyset, \Omega}$, where $\Omega$ is the sample space of the probability space that holds all random variables. We denote by $\probt{\cdot} = \condprob{\cdot}{\cF_{t - 1}}$ and $\Et{\cdot} = \condE{\cdot}{\cF_{t - 1}}$ the conditional probability and expectation operators, respectively, given the past at the beginning of round $t$. Let $\normw{x}{M} = \sqrt{x\T M x}$. Let
\begin{align}
  \!\!\! E_{1, t}
  = \set{\forall i \in [K]:
  \abs{x_i\T \bar{\theta}_t - x_i\T \theta_\ast} \leq
  c_1 \normw{x_i}{G_t^{-1}}}
  \label{eq:theta hat is close}
\end{align}
be the event that $\bar{\theta}_t$ is \say{close} to $\theta_\ast$ in round $t$, where $\bar{\theta}_t$ is defined in \eqref{eq:theta hat} and $c_1 > 0$ is tuned such that $\bar{E}_{1, t}$, the complement of $E_{1, t}$, is unlikely. Let $E_1 = \bigcap_{t = d + 1}^n E_{1, t}$ and $\bar{E}_1$ be its complement. Let
\begin{align}
  \!\!\! E_{2, t}
  = \set{\forall i \in [K]:
  \abs{x_i\T \tilde{\theta}_t - x_i\T \bar{\theta}_t}
  \leq c_2 \normw{x_i}{G_t^{-1}}}
  \label{eq:theta tilde is close}
\end{align}
be the event that $\tilde{\theta}_t$ is \say{close} to $\bar{\theta}_t$ in round $t$, where $\tilde{\theta}_t$ is defined in \eqref{eq:theta tilde} and $c_2 > 0$ is tuned such that $\bar{E}_{2, t}$, the complement of $E_{2, t}$, is unlikely given any past.

\subsection{General Regret Bound}
\label{sec:regret bound}

In this section, we prove a general regret bound for \emph{any} \say{model-based} linear bandit algorithm. The algorithm is model-based if the pulled arm in round $t$ is chosen as in line \ref{alg:linphe:pulled arm} of $\linphe$, where $\tilde{\theta}_t$ can be computed by any possibly randomized procedure based on past data.

Our regret bound involves three probability constants. The first constant, $p_1$, is an upper bound on the probability of event $\bar{E}_1$, that is $p_1 \geq \prob{\bar{E}_1}$. The second constant, $p_2$, is an upper bound on the probability of event $\bar{E}_{2, t}$ given any past,
\begin{align}
  \probt{\bar{E}_{2, t}}
  \leq p_2\,.
  \label{eq:p2}
\end{align}
The last constant, $p_3$, is a lower bound on the probability that the estimated reward of the optimal arm $1$ is optimistic given any past,
\begin{align}
  \probt{x_1\T \tilde{\theta}_t - x_1\T \bar{\theta}_t > c_1 \normw{x_1}{G_t^{-1}}}
  \geq p_3\,.
  \label{eq:p3}
\end{align}
To simplify exposition, we define $\langle x\rangle = \min \set{x, 1}$. The main result of this section is the following regret bound.

\begin{theorem}
\label{thm:lin upper bound} Let $c_1, c_2 \geq 1$. Let $A$ be any algorithm that pulls arm $I_t = \argmax_{i \in [K]} x_i\T \tilde{\theta}_t$ in round $t$, where $\tilde{\theta}_t$ is estimated from past data. Let the rewards be in $[0, 1]$; $p_1$, $p_2$, and $p_3$ be defined as above; and $p_3 > p_2$. Then the expected $n$-round regret of $A$ is bounded as $R(n) \leq$
\begin{align*}
  (c_1 + c_2) \left(1 + \frac{2}{p_3 - p_2}\right) \sqrt{c_3 n} +
  (p_1 + p_2) n + d\,,
\end{align*}
where $c_3$ is defined in \cref{tab:constants}.
\end{theorem}

\cref{thm:lin upper bound} is extracted from prior work, where similar randomized algorithms have been analyzed \cite{agrawal13thompson,valko14spectral}. The proof relies on the following two lemmas.

\begin{lemma}
\label{lem:per-round regret} Let $c_1, c_2 \geq 1$. Then for any round $t > d$ and history $\cF_{t - 1}$, $\Et{\Delta_{I_t} \I{E_{1, t}}} \leq$
\begin{align*}
  (c_1 + c_2) \left(1 + \frac{2}{p_3 - p_2}\right)
  \Et{\xmin{\normw{x_{I_t}}{G_t^{-1}}}} + p_2\,.
\end{align*}
\end{lemma}

We defer the proof of \cref{lem:per-round regret} to \cref{sec:proofs}. We also use Lemma 11 of \citet{abbasi-yadkori11improved}.

\begin{lemma}
\label{lem:sum of squared confidence widths} For any $\lambda > 0$, $\sum_{t = d + 1}^n \xmin{\normw{x_{I_t}}{G_t^{-1}}^2} \leq c_3$, where $c_3 = 2 d \log(1 + n L^2 / (d \lambda))$.
\end{lemma}

\begin{proof}[Proof of \cref{thm:lin upper bound}]
First, we split the regret based on whether event $E_1$ occurs and obtain
\begin{align*}
  R(n)
  & \leq \sum_{t = d + 1}^n \E{\Delta_{I_t}} + d \\
  & \leq \sum_{t = d + 1}^n \E{\Delta_{I_t} \I{E_{1, t}}} + n \prob{\bar{E}_1} + d \\
  & \leq \sum_{t = d + 1}^n \E{\Et{\Delta_{I_t} \I{E_{1, t}}}} + p_1 n + d\,.
\end{align*}
Since $E_{1, t}$ is $\cF_{t - 1}$ measurable, $\Et{\Delta_{I_t} \I{E_{1, t}}}$ can be bounded from above by \cref{lem:per-round regret}. We apply it and get
\begin{align*}
  R(n)
  \leq {} & (c_1 + c_2) \left(1 + \frac{2}{p_3 - p_2}\right) \times {} \\
  & \E{\sum_{t = d + 1}^n \xmin{\normw{x_{I_t}}{G_t^{-1}}}} +
  (p_1 + p_2) n + d\,.
\end{align*}
By the Cauchy-Schwarz inequality and \cref{lem:sum of squared confidence widths},
\begin{align*}
  \sum_{t = d + 1}^n \xmin{\normw{x_{I_t}}{G_t^{-1}}}
  \leq \sqrt{n \sum_{t = d + 1}^n \xmin{\normw{x_{I_t}}{G_t^{-1}}^2}}
  \leq \sqrt{c_3 n}\,.
\end{align*}
The claim follows from chaining the above two inequalities.
\end{proof}

\subsection{Expected $n$-Round Regret of $\linphe$}
\label{sec:discussion}

\begin{table}[t]
  \centering
  {\small
  \begin{tabular}{l l} \hline
    Constant & Value \\ \hline
    $L$ & $\max_{i \in [K]} \normw{x_i}{2}$ \\
    $L_\ast$ & $\normw{\theta_\ast}{2}$ \\
    $p_1$ & $1 / n$ \\
    $p_2$ & $1 / n^2$ \\
    $p_3$ & $\displaystyle \frac{1 / 2 - 128 \, c_1^2 n^{-3}}{16 \log n}$ \\
    $c_1$ & $\displaystyle \frac{1}{2} \sqrt{d \log
    \left(n + n^2 L^2 / (d \lambda)\right)} + \lambda^\frac{1}{2} L_\ast$ \\
    $c_2$ & $\sqrt{a \log(\sqrt{2 K} n)}$ \\
    $c_3$ & $2 d \log(1 + n L^2 / (d \lambda))$ \\
    $\lambda$ & $\lambda_{\min}(G_{d + 1}) / 4$ \\
    $a$ & $\ceils{16 \, c_1^2}$ \\ \hline
  \end{tabular}
  }
  \caption{Summary of the constants in the analysis.}
  \label{tab:constants}
\end{table}

The main result of this section is stated below.

\begin{theorem}
\label{thm:linphe upper bound} Let all parameters be chosen as in \cref{tab:constants} and $n > \max \set{34, \, 8 \sqrt{2} c_1} = \tilde{O}(d)$. Then the expected $n$-round regret of $\linphe$ is $R(n) = \tilde{O}(d \sqrt{n})$.
\end{theorem}

Our regret bound scales with $d$ and $n$ as that of $\lints$ \cite{agrawal13thompson}. This is unsurprising, since we build on the analysis of $\lints$. Our bound also does not improve over those of OFU designs, such as $\linucb$ \cite{abbasi-yadkori11improved}. The improvement is in practical performance, as shown in \cref{sec:experiments}.

The proof of \cref{thm:linphe upper bound} follows from \cref{thm:lin upper bound} for appropriate choice of $c_1$, $c_2$, $p_1$, $p_2$, and $p_3$. These values, together with a number of other constants, are summarized in \cref{tab:constants}. The proof is broken down into lemmas, which are proved in \cref{sec:proofs}.

The first lemma guides our choice of $c_1$. Specifically, we get $p_1 = 1 / n$ for $c_1$ in \cref{tab:constants}.

\begin{lemma}[Least-squares concentration]
\label{lem:theta hat concentration} For any $\lambda > 0$, $\delta > 0$, and
\begin{align*}
  c_1
  = \frac{1}{2} \sqrt{d \log \left(\frac{1 + n L^2 / (d \lambda)}{\delta}\right)} +
  \lambda^\frac{1}{2} L_\ast\,,
\end{align*}
event $E_1$ occurs with probability at least $1 - \delta$.
\end{lemma}

The next lemma, together with the union bound over all arms, guarantees that $p_2 = 1 / n^2$ for $c_2$ in \cref{tab:constants}. This lemma is a key part of our analysis.

\begin{lemma}[Concentration]
\label{lem:theta tilde concentration} For any $t > d$, $c > 0$, and vector $x \in \realset^d$, we have
\begin{align*}
  \probt{\abs{x\T \tilde{\theta}_t - x\T \bar{\theta}_t} \geq
  c \normw{x}{G_t^{-1}}}
  \leq 2 \exp\left[- 2 c^2 / a\right]\,.
\end{align*}
\end{lemma}

The next lemma bounds $p_3$ from below. This lemma is another key part of our analysis.

\begin{lemma}[Anti-concentration]
\label{lem:theta tilde anti-concentration} For any round $t > d$, constants $a$ and $c$ such that $2 a \log n > c^2 > 0$, and vector $x \in \realset^d$ such that $x \neq \mathbf{0}$, we have
\begin{align*}
  & \probt{x\T \tilde{\theta}_t - x\T \bar{\theta}_t > c \normw{x}{G_t^{-1}}} \\
  & \quad \geq \frac{1}{16 \log n} (1 - \lambda \lambda_{\min}^{-1}(G_{d + 1}) -
  4 a^{-1} c^2 - 8 a n^{-3})\,.
\end{align*}
\end{lemma}

For $\lambda = \lambda_{\min}(G_{d + 1}) / 4$, $a = \ceils{16 \, c_1^2}$, $c = c_1$, and any $x_1 \neq \mathbf{0}$, \cref{lem:theta tilde anti-concentration} implies that
\begin{align*}
  p_3 - p_2 \geq
  \frac{1 / 2 - 128 \, c_1^2 n^{-3} - 16 \, n^{-2} \log n}{16 \log n}\,.
\end{align*}
Since $\lambda_{\min}(G_{d + 1}) = \lambda + \lambda_{\min}(\sum_{\ell = 1}^d X_\ell X_\ell\T)$, we have that $\lambda = \lambda_{\min}(\sum_{\ell = 1}^d X_\ell X_\ell\T) / 3$. Finally, we set $c_3$ as in \cref{tab:constants}. Now are ready to prove \cref{thm:linphe upper bound}.

\begin{proof}[Proof of \cref{thm:linphe upper bound}]
If $x_1 = \mathbf{0}$, the proof is trivial. Now suppose that $x_1 \neq \mathbf{0}$. Since $L = O(\sqrt{d})$, $L_\ast = O(\sqrt{d})$, and $\lambda = O(1)$, we have $c_1 = \tilde{O}(\sqrt{d})$. Moreover, since $a = \ceils{16 \, c_1^2}$, we have $c_2 = \tilde{O}(\sqrt{d})$. Finally, $c_3 = \tilde{O}(d)$.

Now we show that $1 + 1 / (p_3 - p_2) = \tilde{O}(1)$. Trivially, $n^{-1} \log n \leq 1$ for $n \geq 1$. In addition, for $n \geq 8 \sqrt{2} c_1$, $128 \, c_1^2 n^{-2} \leq 1$. So, for any such $n$,
\begin{align*}
  p_3 - p_2 \geq
  \frac{1 / 2 - 17 \, n^{-1}}{16 \log n}\,.
\end{align*}
Finally, for any $n > 34$, the above lower bound is positive and $1 + 1 / (p_3 - p_2) = \tilde{O}(1)$. This concludes our proof.
\end{proof}

%% file: Experiments.tex

\section{EXPERIMENTS}
\label{sec:experiments}

\begin{figure*}[t]
  \centering
  \includegraphics[width=6.4in]{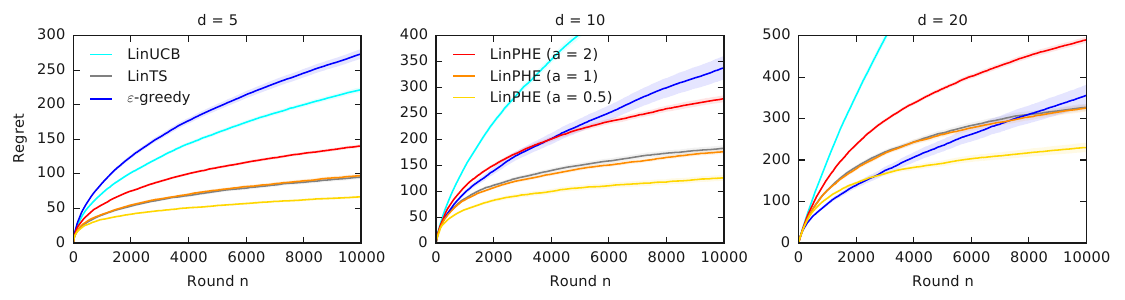}
  \vspace{-0.1in}
  \caption{Comparison of $\linphe$ to several baselines in three linear bandit problems. All results are averaged over $100$ random problem instances. The shaded areas are standard errors of the estimates. The legend is split between the first two plots.}
  \label{fig:linear bandit}
\end{figure*}

We conduct two experiments to evaluate both $\linphe$ and $\logphe$ in terms of their regret. The algorithms are compared to several state-of-the-art baselines.

\subsection{Linear Bandit}
\label{sec:linear bandit experiments}

The first experiment is with linear bandits. We experiment with dimensions $d$ from $5$ to $20$. The number of arms is $K = 100$. To avoid biases, we randomly generate problem instances. Each instance is generated as follows. The first $d - 1$ entries of feature vector $x_i$ are drawn from a unit $(d - 2)$-sphere and the last entry is $1$. The first $d - 1$ entries of parameter vector $\theta_\ast$ are drawn from a $(d - 2)$-sphere with radius $0.5$ and the last entry is $0.5$. This construction guarantees that $x_i\T \theta_\ast \in [0, 1]$ for all arms $i$. The reward of arm $i$ is drawn i.i.d.\ from $\mathrm{Ber}(x_i\T \theta_\ast)$. The horizon is $n = 10000$ rounds and our results are averaged over $100$ problem instances.

We compare $\linphe$ to $\linucb$ \cite{abbasi-yadkori11improved}, $\lints$ \cite{agrawal13thompson}, and the $\varepsilon$-greedy policy \cite{sutton98reinforcement,auer02finitetime} with a linear model. $\linucb$ is an OFU algorithm. We set the regularization parameter in $\linucb$ as $\lambda = 1$. All other parameters are set as in \citet{abbasi-yadkori11improved}. $\lints$ is a posterior sampling algorithm. We set its prior to $\mathcal{N}(0, I_d)$. The exploration rate in the $\varepsilon$-greedy policy is $\varepsilon_t = \min \{1, 0.05 / (2 \sqrt{t})\}$, which results in about $5\%$ exploration. We experiment with three practical perturbation scales $a$ in $\linphe$: $2$, $1$, and $0.5$. We implement $\linphe$ with non-integer perturbation scales $a$ by replacing $B(a \, T_{i, t - 1}, 1 / 2)$ in \cref{sec:efficient implementation} with $B(\ceils{a \, T_{i, t - 1}}, 1 / 2)$.

Our results are reported in \cref{fig:linear bandit}. We observe the following trends. First, $\linphe$ outperforms $\linucb$ at all perturbation scales $a$. Second, $\linphe$ outperforms the $\varepsilon$-greedy policy at all perturbation scales $a$ in the first two problems. In the last problem, this happens only at $a \leq 1$. Finally, $\linphe$ performs similarly to $\lints$ at $a = 1$ and outperforms it at $a = 0.5$. However, the run time of $\linphe$ is less than a half of that of $\lints$. For instance, at $d = 5$, the average run times of $\linphe$ and $\lints$ are $113$ and $273$ seconds, respectively. The increased run time of $\lints$ is due to posterior sampling from the multivariate normal distribution.

\subsection{Logistic Bandit}
\label{sec:logistic bandit experiments}

The last experiment is with logistic bandits (\cref{sec:algorithm logphe}). The experimental setup differs from \cref{sec:linear bandit experiments} only in how $\theta_\ast$ is generated. The first $d - 1$ entries of parameter vector $\theta_\ast$ are drawn from a $(d - 2)$-sphere with radius $1.5$ and the last entry is $0$. By design, $\abs{x_i\T \theta_\ast} \leq 1.5$.

We compare $\logphe$ to $\logts$ \cite{chapelle11empirical,russo18tutorial}, $\glmucb$ \cite{filippi10parametric}, $\ucbglm$ \cite{li17provably}, and the $\varepsilon$-greedy policy \cite{sutton98reinforcement,auer02finitetime} with a logistic model. $\glmucb$ and $\ucbglm$ are OFU methods for logistic bandits. We implement them with regularization (\cref{sec:algorithm logphe}) and $\lambda = 1$. The minimum derivative of the mean function, which is tunable in both methods, is set to the most optimistic value of $1 / 4$. All other parameters are set as suggested by theory. $\logts$ is a posterior sampling algorithm for logistic regression, which uses the Laplace approximation and has prior $\mathcal{N}(0, I_d)$. The $\varepsilon$-greedy policy is implemented as in \cref{sec:linear bandit experiments}.

Our results are reported in \cref{fig:logistic bandit}. We observe similar trends to \cref{sec:linear bandit experiments}. In particular, $\logphe$ usually outperforms OFU algorithms and is competitive with posterior sampling when $a \leq 1$. In summary, our experimental results show that both $\linphe$ and $\logphe$ perform well, and are comparable to or better than existing algorithms.

\begin{figure*}[t]
  \centering
  \includegraphics[width=6.4in]{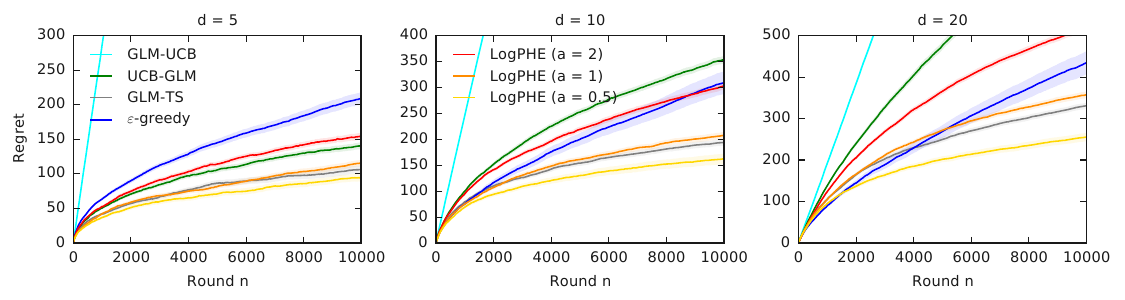}
  \vspace{-0.1in}
  \caption{Comparison of $\logphe$ to several baselines in three logistic bandit problems. All results are averaged over $100$ random problem instances. The shaded areas are standard errors of the estimates. The legend is split between the first two plots.}
  \label{fig:logistic bandit}
\end{figure*}

%% file: RelatedWork.tex

\section{RELATED WORK}
\label{sec:related work}

Our work is motivated by \citet{kveton19perturbed}, who proposed a multi-armed bandit algorithm that pulls the arm with the highest average reward in its perturbed history with i.i.d.\ pseudo-rewards. We generalize this approach to linear, and more generally contextual, bandits. This generalization is important. While the perturbed history is conceptually simple, it is unclear how to extend it to structured problems, and assessing if such a generalization is sound is non-trivial. We propose one generalization, and show it to be both sound and effective.

Our work is related to posterior sampling. In particular, let $\mu \sim \mathcal{N}(\mu_0, \sigma^2)$ and $(Y_\ell)_{\ell = 1}^s \sim \mathcal{N}(\mu, \sigma^2)$ be $s$ i.i.d.\ noisy observations of $\mu$. Then the posterior distribution of $\mu$ conditioned on $(Y_\ell)_{\ell = 1}^s$ is
\begin{align}
  \mathcal{N}\left(\frac{\mu_0 + \sum_{\ell = 1}^s Y_\ell}{s + 1}, \,
  \frac{\sigma^2}{s + 1}\right)\,.
  \label{eq:normal posterior}
\end{align}
It is easy to see that the above distribution can be indirectly sampled from as follows. First, draw $s + 1$ i.i.d.\ samples $(Z_\ell)_{\ell = 0}^s \sim \mathcal{N}(0, \sigma^2)$. Then
\begin{align*}
  \frac{\mu_0 + \sum_{\ell = 1}^s Y_\ell + \sum_{\ell = 0}^s Z_\ell}{s + 1}
\end{align*}
is a sample from \eqref{eq:normal posterior}. This equivalence can be generalized to linear models with Gaussian noise \cite{lu17ensemble}. Unfortunately, it holds only for normal random variables. Therefore, it cannot justify our perturbation scheme as a form of posterior sampling.

The design of $\linphe$ is similar to \emph{follow the perturbed leader (FPL)} \cite{hannan57approximation,kalai05efficient}. FPL has been traditionally studied in the \emph{non-stochastic full-information} setting. \citet{neu13efficient} extended it to \emph{semi-bandits} using geometric resampling. Their algorithm cannot solve our problem efficiently because it is for a $K$-armed bandit with independent arms.

Our work is closely related to bootstrapping exploration \cite{baransi14subsampling,eckles14thompson,osband15bootstrapped,tang15personalized,elmachtoub17practical,kveton19garbage,vaswani18new}, where the learning agent perturbs its history of observations by resampling in order to achieve exploration. Contextual bootstrapping algorithms \cite{tang15personalized,elmachtoub17practical,kveton19garbage,vaswani18new} have superior empirical performance but no regret bounds. Our work provides a stepping stone for the analysis of such algorithms, since our perturbation scheme is similar but simpler.

%% file: Conclusions.tex

\section{CONCLUSIONS}
\label{sec:conclusions}

We propose $\linphe$, a new online algorithm for cumulative regret minimization in stochastic linear bandits. The key idea in $\linphe$ is to perturb the history in round $t$ by $O(t)$ i.i.d.\ pseudo-rewards, which are drawn from the maximum variance distribution. We derive a $\tilde{O}(d \sqrt{n})$ bound on the $n$-round regret of $\linphe$, where $d$ is the number of features. We also propose $\logphe$, a natural generalization of $\linphe$ to a logistic model. We evaluate $\linphe$ and $\logphe$ empirically. Both algorithms are competitive with Thompson sampling, although they are derived based on different insights.

$\linphe$ can be easily extended to any linear model with a bounded support. In particular, if $Y_{i, t} \in [m, M]$, $Y_\ell$ in $\linphe$ should be replaced with $(Y_\ell - m) / (M - m)$.

Our work can be extended in several directions. First, although we propose $\logphe$ for a logistic model, we do not analyze it. We believe that the regret analysis is possible because generalized linear bandit analyses \cite{filippi10parametric,li17provably} are similar to linear bandit analyses \cite{dani08stochastic,abbasi-yadkori11improved}. Second, the theory-suggested perturbation scale $a$ in \cref{tab:constants} is too conservative to be practical, for the same reason as the analyzed variant of $\lints$ in \citet{agrawal13thompson}. A tighter analysis should be possible. Third, our key technical lemmas, \cref{lem:theta tilde concentration,lem:theta tilde anti-concentration}, can be extended to other randomized pseudo-rewards than Bernoulli. This would be necessary for other generalized linear models than logistic. Finally, our design seems conservative since the strategy for adding pseudo-rewards does not adapt over time. More adaptive designs may be possible.

%% file: Appendix.tex

\appendix

\section{PROOFS}
\label{sec:proofs}

\subsection{Proof of \cref{lem:per-round regret}}

Let
\begin{align}
  \bar{S}_t
  = \set{i \in [K]: (c_1 + c_2) \normw{x_i}{G_t^{-1}} \geq \Delta_i}
  \label{eq:undersampled arms}
\end{align}
be the set of \emph{undersampled arms} in round $t$. Note that by definition $1 \in \bar{S}_t$. The set of \emph{sufficiently sampled arms} is defined as $S_t = [K] \setminus \bar{S}_t$. Let
\begin{align}
  \textstyle
  J_t
  = \argmin_{i \in \bar{S}_t} \normw{x_i}{G_t^{-1}}
  \label{eq:lest uncertain undersampled arm}
\end{align}
be the \emph{least uncertain undersampled arm} in round $t$. In all steps below, we assume that event $E_{1, t}$ occurs.

Let $c = c_1 + c_2$. In round $t$ on event $E_{2, t}$,
\begin{align*}
  \Delta_{I_t}
  = {} & \Delta_{J_t} + \xmin{x_{J_t}\T \theta_\ast - x_{I_t}\T \theta_\ast} \\
  \leq {} & \Delta_{J_t} +
  \xmin{x_{J_t}\T \tilde{\theta}_t - x_{I_t}\T \tilde{\theta}_t} + {} \\
  & c \, (\xmin{\normw{x_{I_t}}{G_t^{-1}}} +
  \xmin{\normw{x_{J_t}}{G_t^{-1}}}) \\
  \leq {} & c \, (\xmin{\normw{x_{I_t}}{G_t^{-1}}} +
  2 \xmin{\normw{x_{J_t}}{G_t^{-1}}})\,,
\end{align*}
where the first inequality is by the definitions of events $E_{1, t}$ and $E_{2, t}$, and the second follows from the definitions of $I_t$ and $J_t$. We also used that $c = c_1 + c_2 \geq 1$. Now we take the expectation of both sides and get
\begin{align*}
  & \Et{\Delta_{I_t}} \\
  & \quad = \Et{\Delta_{I_t} \I{E_{2, t}}} +
  \Et{\Delta_{I_t} \I{\bar{E}_{2, t}}} \\
  & \quad \leq c \, \Et{\xmin{\normw{x_{I_t}}{G_t^{-1}}} +
  2 \xmin{\normw{x_{J_t}}{G_t^{-1}}}} + \probt{\bar{E}_{2, t}}\,.
\end{align*}
The last step is to bound $\Et{\xmin{\normw{x_{J_t}}{G_t^{-1}}}}$ from above. The key observation is that
\begin{align*}
  & \Et{\xmin{\norm{x_{I_t}}_{G_t^{-1}}}} \\
  & \quad \geq \Et{\xmin{\norm{x_{I_t}}_{G_t^{-1}}} \,\middle|\, I_t \in \bar{S}_t}
  \probt{I_t \in \bar{S}_t} \\
  & \quad \geq \xmin{\norm{x_{J_t}}_{G_t^{-1}}} \, \probt{I_t \in \bar{S}_t}\,,
\end{align*}
where the last inequality is from the definition of $J_t$ and that $\bar{S}_t$ is $\cF_{t - 1}$-measurable. We rearrange the inequality and get
\begin{align*}
  \xmin{\norm{x_{J_t}}_{G_t^{-1}}}
  \leq \Et{\xmin{\norm{x_{I_t}}_{G_t^{-1}}}} \Big/ \
  \probt{I_t \in \bar{S}_t}\,.
\end{align*}
Next we bound $\probt{I_t \in \bar{S}_t}$ from below. On event $E_{1, t}$,
\begin{align*}
  \probt{I_t \in \bar{S}_t}
  & \geq \probt{\exists i \in \bar{S}_t: x_i\T \tilde{\theta}_t >
  \max_{j \in S_t} x_j\T \tilde{\theta}_t} \\
  & \geq \probt{x_1\T \tilde{\theta}_t >
  \max_{j \in S_t} x_j\T \tilde{\theta}_t} \\
  & \geq \probt{x_1\T \tilde{\theta}_t >
  \max_{j \in S_t} x_j\T \tilde{\theta}_t, \, E_{2, t} \text{ occurs}} \\
  & \geq {} \probt{x_1\T \tilde{\theta}_t >
  x_1\T \theta_\ast, \, E_{2, t} \text{ occurs}} \\
  & \geq \probt{x_1\T \tilde{\theta}_t > x_1\T \theta_\ast} - \probt{\bar{E}_{2, t}}\,.
\end{align*}
Note that we require a sharp inequality because $x_i\T \tilde{\theta}_t \geq \max_{j \in S_t} x_j\T \tilde{\theta}_t$ does not imply that arm $i$ is pulled. The fourth inequality holds because for any $j \in S_t$,
\begin{align*}
  x_j\T \tilde{\theta}_t
  \leq x_j\T \theta_\ast + c \normw{x_j}{G_t^{-1}}
  < x_j\T \theta_\ast + \Delta_j
  = x_1\T \theta_\ast
\end{align*}
on event $E_{1, t} \cap E_{2, t}$. Finally,
\begin{align*}
  \probt{x_1\T \tilde{\theta}_t > x_1\T \theta_\ast} 
  \geq \probt{x_1\T \tilde{\theta}_t -
  x_1\T \bar{\theta}_t > c_1 \normw{x_1}{G_t^{-1}}}
\end{align*}
on event $E_{1, t}$, because $x_1\T \theta_\ast \leq x_1\T \bar{\theta}_t + c_1 \normw{x_1}{G_t^{-1}}$ holds on event $E_{1, t}$. Now we chain all inequalities and use the definitions of $p_1$, $p_2$, and $p_3$ to complete the proof.

\subsection{Proof of \cref{lem:theta hat concentration}}

By the Cauchy-Schwarz inequality,
\begin{align*}
  x_i\T \bar{\theta}_t - x_i\T \theta_\ast
  & = x_i\T G_t^{- \frac{1}{2}} G_t^\frac{1}{2} (\bar{\theta}_t - \theta_\ast) \\
  & \leq \|\bar{\theta}_t - \theta_\ast\|_{G_t} \normw{x_i}{G_t^{-1}}\,.
\end{align*}
Now note that the least-squares estimate $\bar{\theta}_t$ is computed from sub-Gaussian rewards with variance proxy $1 / 4$. As a result, by Theorem 2 of \citet{abbasi-yadkori11improved} for $R = 1 / 2$, $\|\bar{\theta}_t - \theta_\ast\|_{G_t} \leq c_1$ holds jointly in all rounds $t \leq n$ with probability of at least $1 - \delta$. This concludes the proof.

\subsection{Proof of \cref{lem:theta tilde concentration}}

Let
\begin{align*}
  U
  & = \sum_{\ell = 1}^{t - 1} \sum_{j = 1}^a
  x\T G_t^{-1} X_\ell Z_{j, \ell}\,, \\
  \bar{U}
  & = \sum_{\ell = 1}^{t - 1} \sum_{j = 1}^a
  x\T G_t^{-1} X_\ell \bar{Z}_{j, \ell}\,,
\end{align*}
and $D = U - \bar{U}$. Then by Hoeffding's inequality,
\begin{align*}
  & \probt{\abs{x\T \tilde{\theta}_t - x\T \bar{\theta}_t} \geq
  c \normw{x}{G_t^{-1}}} \\
  & \quad = \probt{\abs{D} \geq c \normw{x}{G_t^{-1}}} \\
  & \quad \leq 2 \exp\left[- \frac{2 c^2 \normw{x}{G_t^{-1}}^2}
  {a \sum_{\ell = 1}^{t - 1} x\T G_t^{-1} X_\ell X_\ell\T G_t^{-1} x}\right]\,.
\end{align*}
This step of the proof relies on the fact that new $Z_{j, \ell}$ are generated in each round $t$. Also note that
\begin{align}
  & \sum_{\ell = 1}^{t - 1} x\T G_t^{-1} X_\ell X_\ell\T G_t^{-1} x
  \label{eq:cancel G inverse} \\
  & \quad \leq x\T G_t^{-1} \left(\sum_{\ell = 1}^{t - 1}
  X_\ell X_\ell\T + \lambda I_d\right) G_t^{-1} x
  = \normw{x}{G_t^{-1}}^2\,.
  \nonumber
\end{align}
Our claim follows from chaining all above inequalities.

\subsection{Proof of \cref{lem:theta tilde anti-concentration}}

Let $U$, $\bar{U}$, and $D$ be defined as in the proof of \cref{lem:theta tilde concentration}. Then $x\T \tilde{\theta}_t - x\T \bar{\theta}_t = D$. We also define events
\begin{align*}
  F_1
  & = \set{\abs{D} \leq c \normw{x}{G_t^{-1}}}\,, \\
  F_2
  & = \set{\abs{D} \leq \sqrt{2 a \log n} \normw{x}{G_t^{-1}}}\,.
\end{align*}
Since $2 a \log n > c^2$, $F_1 \subset F_2$. Then
\begin{align*}
  \condvar{U}{\cF_{t - 1}}
  = {} & \Et{D^2 \I{F_1}} + {} \\
  & \Et{D^2 \I{\bar{F}_1, F_2}} + {} \\
  & \Et{D^2 \I{\bar{F}_2}}\,.
\end{align*}
Now we bound each term on the right-hand side of the above equality from above. From the definition of event $F_1$, term $1$ is bounded as
\begin{align*}
  \Et{D^2 \I{F_1}}
  \leq c^2 \normw{x}{G_t^{-1}}^2\,.
\end{align*}
By the definition of $F_1$ and $F_2$, term $2$ is bounded as
\begin{align*}
  & \Et{D^2 \I{\bar{F}_1, F_2}} \\
  & \quad \leq (2 a \normw{x}{G_t^{-1}}^2 \log n) \,
  \probt{\bar{F}_1, F_2 \text{ occur}} \\
  & \quad \leq (2 a \normw{x}{G_t^{-1}}^2 \log n) \,
  \probt{\abs{D} > c \normw{x}{G_t^{-1}}}\,.
\end{align*}
Now we bound term $3$. First, note that
\begin{align*}
  \abs{D}
  & \leq a \sum_{\ell = 1}^{t - 1} \abs{x\T G_t^{-1} X_\ell} \\
  & \leq a \sqrt{n} \sqrt{\sum_{\ell = 1}^{t - 1}
  x\T G_t^{-1} X_\ell X_\ell\T G_t^{-1} x} \\
  & \leq a \sqrt{n} \normw{x}{G_t^{-1}}\,,
\end{align*}
where the last step follows from \eqref{eq:cancel G inverse}. Then, by the definition of event $F_2$ and \cref{lem:theta tilde concentration} for $c = \sqrt{2 a \log n}$,
\begin{align*}
  \Et{D^2 \I{\bar{F}_2}} 
  & \leq a^2 n \normw{x}{G_t^{-1}}^2 \probt{\bar{F}_2} \\
  & \leq \frac{2 a^2 \normw{x}{G_t^{-1}}^2}{n^3}\,.
\end{align*}
Finally, by the definition of $U$,
\begin{align*}
  \condvar{U}{\cF_{t - 1}}
  & = \frac{a}{4} \sum_{\ell = 1}^{t - 1}
  x\T G_t^{-1} X_\ell X_\ell\T G_t^{-1} x \\
  & = \frac{a}{4} \normw{x}{G_t^{-1}}^2 -
  \frac{a}{4} \lambda x\T G_t^{-2} x\,.
\end{align*}
We bound the last term from below as follows. For any positive semi-definite matrix $M \in \realset^{d \times d}$,
\begin{align*}
  x\T M^2 x
  & = \lambda_{\max}^2(M) \, x\T \left(\lambda_{\max}^{-2}(M) M^2\right) x \\
  & \leq \lambda_{\max}^2(M) \, x\T \left(\lambda_{\max}^{-1}(M) M\right) x \\
  & = \lambda_{\max}(M) \normw{x}{M}^2\,,
\end{align*}
where the inequality follows from the fact that all eigenvalues of $\lambda_{\max}^{-2}(M) M^2$ are in $[0, 1]$. We apply this upper bound for $M = G_t^{-1}$ and get that
\begin{align*}
  \condvar{U}{\cF_{t - 1}}
  & \geq \frac{a}{4} \normw{x}{G_t^{-1}}^2 -
  \frac{a \lambda}{4 \, \lambda_{\min}(G_t)} \normw{x}{G_t^{-1}}^2 \\
  & \geq \frac{a}{4} \normw{x}{G_t^{-1}}^2 -
  \frac{a \lambda}{4 \, \lambda_{\min}(G_{d + 1})} \normw{x}{G_t^{-1}}^2\,,
\end{align*}
where the last inequality is by $\lambda_{\min}(G_t) \geq \lambda_{\min}(G_{d + 1})$ and holds for any $t > d$.

Now we combine all above inequalities and get
\begin{align*}
  & \left[\frac{a}{4} - \frac{a \lambda}{4 \, \lambda_{\min}(G_{d + 1})} -
  c^2 - \frac{2 a^2}{n^3}\right] \normw{x}{G_t^{-1}}^2 \\
  & \quad \leq (2 a \normw{x}{G_t^{-1}}^2 \log n) \,
  \probt{\abs{D} > c \normw{x}{G_t^{-1}}}\,.
\end{align*}
Since $2 a \log n > 0$ and $\normw{x}{G_t^{-1}} > 0$, the above inequality can be simplified as
\begin{align*}
  & \probt{\abs{D} > c \normw{x}{G_t^{-1}}} \\
  & \quad \geq \frac{1}{8 \log n} (1 - \lambda \lambda_{\min}^{-1}(G_{d + 1}) -
  4 a^{-1} c^2 - 8 a n^{-3})\,.
\end{align*}
Finally, we note that the distribution of $D$ is symmetric. Therefore, for any $\eps > 0$, $\probt{\abs{D} > \eps} = 2 \probt{D > \eps}$. This completes the proof.